\documentclass{article}

\usepackage[preprint]{neurips_2025}

\usepackage[utf8]{inputenc} 
\usepackage[T1]{fontenc}    
\usepackage{hyperref}       
\usepackage{url}            
\usepackage{booktabs}       
\usepackage{amsfonts}       
\usepackage{nicefrac}       
\usepackage{microtype}      
\usepackage{graphicx}
\usepackage{wrapfig}
\usepackage{adjustbox}
\usepackage{amsmath}
\usepackage{tikz}
\usepackage{multirow}
\usepackage{arydshln}
\usepackage{colortbl}  
\usepackage{booktabs}

\usepackage{amssymb}
\usepackage{mathtools}
\usepackage{pgfplots}
\usepackage{pifont}

\usetikzlibrary{positioning,matrix,calc,arrows.meta,decorations.pathreplacing,shapes.geometric}
\definecolor{mygreen}{RGB}{46,139,87}
\definecolor{myred}{RGB}{255,152,150}
\definecolor{myblue}{RGB}{30,144,255}
\definecolor{myyellow}{RGB}{219,219,141}
\definecolor{mybrown}{RGB}{197,157,148}
\definecolor{my_black}{HTML}{add2c9}

\usepackage{pgfplotstable}
\usepackage{bm, dsfont}

\newcommand{\F}{\mathrm{F}}

\newcommand*{\E}{\mathbb{E}}

\makeatletter
\newcommand{\ve}{\@ifnextchar\bgroup{\velong}{{\bm{e}}}}
\newcommand{\velong}[1]{{\bm{#1}}}
\makeatother

\usepackage{microtype}
\usepackage{subfigure}
\usepackage{booktabs} 
\usepackage{subcaption}
\usepackage{caption}

\makeatletter
\def\UrlAlphabet{%
      \do\a\do\b\do\c\do\d\do\e\do\f\do\g\do\h\do\i\do\j%
      \do\k\do\l\do\m\do\n\do\o\do\p\do\q\do\r\do\s\do\t%
      \do\u\do\v\do\w\do\x\do\y\do\z\do\A\do\B\do\C\do\D%
      \do\E\do\F\do\G\do\H\do\I\do\J\do\K\do\L\do\M\do\N%
      \do\O\do\P\do\Q\do\R\do\S\do\T\do\U\do\V\do\W\do\X%
      \do\Y\do\Z}
\def\UrlDigits{\do\1\do\2\do\3\do\4\do\5\do\6\do\7\do\8\do\9\do\0}
\g@addto@macro{\UrlBreaks}{\UrlOrds}
\g@addto@macro{\UrlBreaks}{\UrlAlphabet}
\g@addto@macro{\UrlBreaks}{\UrlDigits}

\usepackage{amsmath}
\usepackage{amssymb}
\usepackage{mathtools}
\usepackage{amsthm}
\usepackage{algorithm,algorithmic}
\usepackage{tabularx}
\usepackage{multirow}
\usepackage{rotating}
\usepackage{pifont}
\usepackage{bbding} 
\usepackage{xspace}
\usepackage{makecell}
\usepackage{colortbl}

\usepackage[capitalize,noabbrev]{cleveref}

\usepackage{multirow}
\usepackage{colortbl}
\usepackage[linesnumbered,ruled,noend,algo2e]{algorithm2e}
\usepackage{tikz}
\usetikzlibrary{tikzmark,calc}

\usepackage{xspace}
\newcommand{\our}{\textsc{DevFT}\xspace}

\title{Learning Like Humans: Resource-Efficient Federated Fine-Tuning through Cognitive Developmental Stages}

\author{
  Yebo Wu\textsuperscript{1\dag}, Jingguang Li\textsuperscript{1\dag},
  Zhijiang Guo\textsuperscript{2,3}\thanks{Corresponding Authors. \textsuperscript{\dag} Equal Contribution.},
  Li Li\textsuperscript{1}\footnotemark[1] \\
  \textsuperscript{1}University of Macau, \textsuperscript{2} HKUST, \textsuperscript{3} HKUST (Guangzhou)\\
  \texttt{\{yc37926,mc45005,llili\}@um.edu.mo}, \texttt{zhijiangguo@hkust-gz.edu.cn}
}

\begin{document}

\maketitle

\begin{abstract}

Federated fine-tuning enables Large Language Models (LLMs) to adapt to downstream tasks while preserving data privacy, but its resource-intensive nature limits deployment on edge devices. In this paper, we introduce \textbf{Developmental Federated Tuning (\our)}, a resource-efficient approach inspired by cognitive development that progressively builds a powerful LLM from a compact foundation. \our decomposes the fine-tuning process into developmental stages, each optimizing submodels with increasing parameter capacity. Knowledge from earlier stages transfers to subsequent submodels, providing optimized initialization parameters that prevent convergence to local minima and accelerate training. This paradigm mirrors human learning, gradually constructing comprehensive knowledge structure while refining existing skills. To efficiently build stage-specific submodels, \our introduces deconfliction-guided layer grouping and differential-based layer fusion to distill essential information and construct representative layers. Evaluations across multiple benchmarks demonstrate that \our significantly outperforms state-of-the-art methods, achieving up to 4.59$\times$ faster convergence, 10.67$\times$ reduction in communication overhead, and 9.07\% average performance improvement, while maintaining compatibility with existing approaches. 

\end{abstract}
\section{Introduction}


Large Language Models (LLMs) exhibit exceptional capabilities across diverse domains~\citep{karanikolas2023large,tian2024hydralora}. While fine-tuning effectively adapts these models to specific tasks~\citep{han2024parameter}, it requires substantial task-specific data. This data often resides privately on edge devices, making centralized collection impractical~\citep{wu2024heterogeneity,wu2024neulite,wang2023fedins2}. Federated fine-tuning~\citep{zhang2024towards} offers a privacy-preserving alternative for collaborative adaptation. Nevertheless, deploying massive LLMs for federated fine-tuning on resource-limited edge devices remains challenging due to hardware and communication constraints~\citep{tam2024fedhybrid, tian2024breaking}.

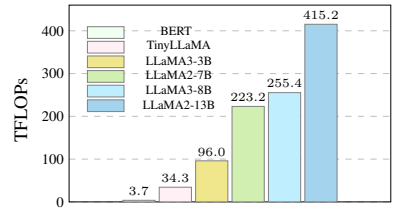
\begin{wrapfigure}{r}{0.37\textwidth}
\vspace{-12pt}
\centering
\adjustbox{max width=0.37\textwidth}{\definecolor{ublue}{RGB}{52,152,219}
\definecolor{ured}{RGB}{240,100,100}
\definecolor{uorange}{RGB}{247,175,89}
\definecolor{upurple}{RGB}{148,137,250}
\definecolor{pink1}{HTML}{FFF0F6}
\definecolor{pink2}{HTML}{FCC2D7}
\definecolor{pink3}{HTML}{FAA2CA}

\definecolor{c1}{HTML}{FFF0F5}
\definecolor{c2}{HTML}{F0E68C}
\definecolor{c3}{HTML}{D2EFB2}
\definecolor{c4}{HTML}{BFEFFF}

\definecolor{c5}{HTML}{A4D3EE}
\definecolor{c6}{HTML}{F0FFF0}

\hspace{-1.2cm} 
\begin{tikzpicture}
    \centering
    \scriptsize{
    \begin{axis}[
    at={(0,0)},
    ymajorgrids,
    grid style=dashed,
    legend style={at={(0.2,1)}, anchor=south west},
    legend cell align={left},
    ybar,
    enlarge x limits=0.8,
    xtick align=inside,
    height=0.35\textwidth,
    width=0.5\textwidth,
    bar width=2.2em,
    xshift=-1.5em,
    nodes near coords,
    nodes near coords align={vertical},
    nodes near coords style={font=\tiny, scale=1,/pgf/number format/fixed, /pgf/number format/precision=1},
    every node near coord/.append style={
    /pgf/number format/.cd,
    fixed,
    fixed zerofill,
    precision=1
    },
    xlabel={},
    xmax=0,
    xmin=0,
    symbolic x coords={0},
    legend style={cells={align=left}},
    xtick=data,
    nodes near coords align={vertical},
    ymin=0,
    ymax=460,
    ytick={0,100,200,300,400},
    yticklabels={0,100,200,300,400},
    xticklabels={},
    xtick style={draw=none},
    ytick style={draw=none},
    yticklabel style={/pgf/number format/fixed,/pgf/number format/fixed},
    legend style={draw=none,
    line width=1pt,
    at={(0.5,1.0)},
    anchor=south},
    xtick=data,
    axis on top=false,
    ]
    \addplot[fill=c6,draw=gray, area legend] coordinates {(0,3.7)};
    \addplot[fill=c1, draw=gray, area legend] coordinates {(0,34.3)};
    \addplot[ fill=c2,draw=gray, area legend] coordinates {(0,96)};
    \addplot[ fill=c3,draw=gray, area legend] coordinates {(0,223.2)};
    \addplot[ fill=c4,draw=gray, area legend] coordinates {(0,255.4)};
    \addplot[ fill=c5,draw=gray, area legend] coordinates {(0,415.2)};
    \end{axis}
    }
    \node [rectangle,draw=gray,fill=c6,inner sep=2pt,minimum height=0.8em,minimum width=2.5em,font=\small,anchor=north,align=center,] (label1) at (1em,12em){};
    \node [rectangle,draw=gray,fill=c1,inner sep=2pt,minimum height=0.8em,minimum width=2.5em,font=\small,anchor=north,align=center,] (label2) at (1em,11em){};
    \node [rectangle,draw=gray,,fill=c2,inner sep=2pt,minimum height=0.8em,minimum width=2.5em,font=\small,anchor=north,align=center,] (label3) at (1em,10em){};
    \node [rectangle,draw=gray,,fill=c3,inner sep=2pt,minimum height=0.8em,minimum width=2.5em,font=\small,anchor=north,align=center,] (label4) at (1em,9em){};
    \node [rectangle,draw=gray,,fill=c4,inner sep=2pt,minimum height=0.8em,minimum width=2.5em,font=\small,anchor=north,align=center,] (label5) at (1em,8em){};
    \node [rectangle,draw=gray,,fill=c5,inner sep=2pt,minimum height=0.8em,minimum width=2.5em,font=\small,anchor=north,align=center,] (label6) at (1em,7em){};
    \node [align=center] (label1_1) at ([xshift=4.8em,yshift=-0.35em]label1.north){\tiny BERT};
    \node [align=center] (label1_2) at ([xshift=5em,yshift=-0.35em]label2.north){\tiny TinyLLaMA};
    \node [align=center] (label1_3) at ([xshift=5em,yshift=-0.35em]label3.north){\tiny LLaMA3-3B};
    \node [align=center] (label1_4) at ([xshift=5em,yshift=-0.35em]label4.north){\tiny LLaMA2-7B};
    \node [align=center] (label1_5) at ([xshift=5em,yshift=-0.35em]label5.north){\tiny LLaMA3-8B};
    \node [align=center] (label1_6) at ([xshift=5em,yshift=-0.35em]label6.north){\tiny LLaMA2-13B};
    
    \node [rotate=90]at (-4.8em,6.3em) {\footnotesize{TFLOPs}};
\end{tikzpicture}}
\caption{Computational overhead in one-step fine-tuning of different language models using LoRA.}
\vspace{-4mm}
\label{fig_motivation_flops}
\end{wrapfigure}
To address these challenges, researchers have proposed various parameter-efficient federated fine-tuning approaches~\cite{wu2025survey}, with LoRA-based methods garnering significant attention due to their efficiency and flexibility~\cite{guo2024selective, wang2024flora}.
However, existing LoRA-based methods typically fine-tune LLMs end-to-end, which remains computationally prohibitive for edge devices compared to small language models such as BERT~\cite{devlin2018bert}.
To quantitatively illustrate this challenge, Figure~\ref{fig_motivation_flops} presents a comparative analysis of the computational costs involved in one-step fine-tuning of various LLaMA~\cite{llama2} series models and BERT.
We observe that even fine-tuning the relatively compact TinyLLaMA~\cite{zhang2024tinyllama} requires 9.3$\times$ more FLOPs than BERT. For the larger LLaMA2-13B~\cite{llama2}, the computational demands surge to 415.2 TFLOPs, which is 112.2$\times$ that of BERT.
Such substantial computational requirements fundamentally challenge the practical deployment of federated fine-tuning on resource-constrained devices, even with current parameter-efficient techniques.

\begin{figure}[!t]
  \centering
  \includegraphics[width=0.88\linewidth]{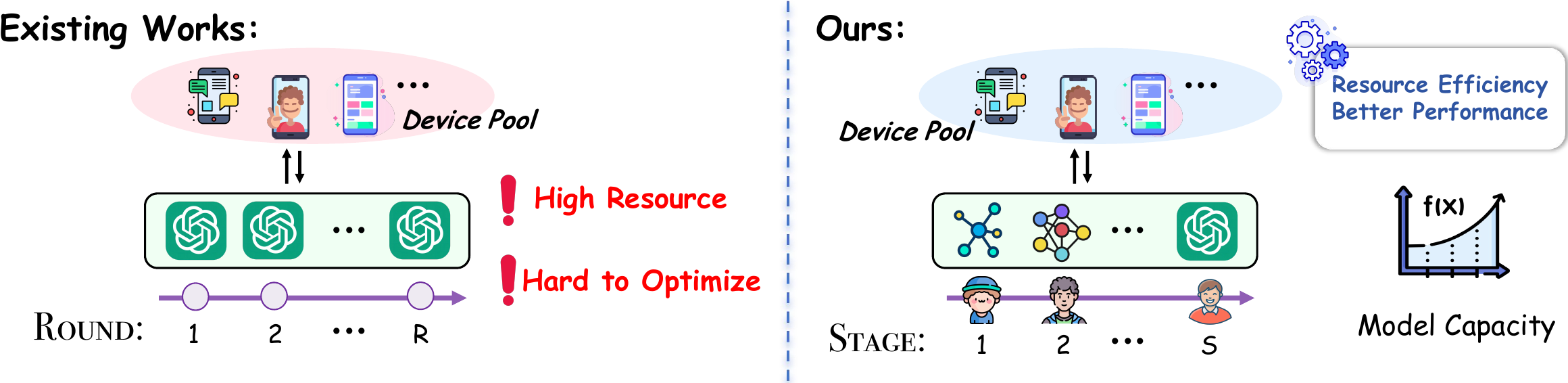}
  \caption{Workflow comparison between existing works and our proposed method.}
    \vspace{-12pt}
\label{fig_motivation}
\end{figure}

Inspired by human cognitive development~\cite{bengio2009curriculum,sweller2008human,mcardle2014human}, where learning progresses incrementally rather than instantaneously, we propose \textbf{Developmental Federated Tuning (\our)}, a resource-efficient federated fine-tuning approach designed to mitigate these computational burdens by progressively cultivating a more capable LLM from a compact foundation. As shown in Figure~\ref{fig_motivation}, rather than continuously updating the LLM throughout the federated fine-tuning process, we decompose the fine-tuning process into distinct developmental stages. Specifically, the learning journey begins with a compact submodel (\textit{analogous to child}), and upon mastering stage-specific competencies, we strategically expand the submodel capacity (\textit{mimicking human growth}), while transferring the acquired knowledge to initialize the submodel of the next stage. This growth process continues until the model reaches its target capacity (\textit{analogous to adult}).

This developmental paradigm, starting with compact models, offers several inherent advantages. Smaller models typically exhibit smoother loss landscapes, thereby effectively mitigating convergence to local minima. Additionally, the insights gained from training smaller models provide an informed initialization for larger architectures, improving overall model performance in subsequent stages. Compared to end-to-end LLM fine-tuning, \our's progressive increase in model capacity significantly accelerates the federated fine-tuning process while reducing computation and communication costs. However, a critical challenge lies in: \textit{How to architect stage-specific submodels to ensure effective knowledge transfer across consecutive stages while optimizing overall performance?}

To address this challenge effectively, \our introduces two novel techniques. The deconfliction-guided layer grouping mechanism initially clusters layers based on parameter similarity, thereby grouping layers with minimal parameter conflicts together. Subsequently, the differential-based layer fusion strategy strategically distills and integrates the unique semantic information of each layer through arithmetic operations, yielding a representative layer for each group that encapsulates the group's collective intelligence and core functionality. These representative layers are then concatenated sequentially to construct the stage-specific submodel for federated fine-tuning. Due to the functional homogeneity within groups, layers can directly inherit knowledge from their corresponding representative layers, thereby facilitating seamless knowledge transfer across stages.

In order to empirically validate the effectiveness of \our and its advantages, we conduct extensive experiments on multiple benchmarks. \our significantly outperforms state-of-the-art methods, achieving up to 4.59$\times$ faster convergence, 10.67$\times$ reduction in communication overhead, and 9.07\% average performance improvement, while maintaining compatibility with existing approaches.

\section{Background}

\subsection{Existing Parameter-Efficient Federated Fine-tuning}

Parameter-efficient federated fine-tuning presents a compelling strategy to mitigate resource demands in distributed learning by freezing most pre-trained model parameters and selectively updating only a small, task-specific subset~\cite{wu2025survey}. These methods generally fall into the following categories. Prompt-based techniques~\cite{guo2023promptfl, yang2023efficient, su2024federated} utilize carefully designed soft prompts to guide model behavior without altering the pre-trained weights. Adapter-based methods~\cite{cai2022fedadapter, kim2023client, liu2023communication, li2022federated} incorporate lightweight adapter layers into the network architecture, allowing for task adaptation with minimal modifications. Notably, LoRA-based approaches have garnered significant interest due to their effectiveness~\cite{guo2024selective}. 

LoRA-based methods~\cite{wang2024flora,sun2024improving} introduce low-rank adaptations to the weight updates, efficiently preserving the expressiveness of the original model. HETLoRA~\cite{cho2024heterogeneous} assigns varying LoRA ranks to different devices to accommodate heterogeneous computational resources effectively. FeDeRA~\cite{yan2024federa} addresses data heterogeneity by initializing LoRA matrices using singular value decomposition on pre-trained parameters. FLASC~\cite{kuo2024federated} introduces sparsity into LoRA modules to decrease communication overhead. While these methods have shown promise in their respective domains, they often do not fully tackle the fundamental issue of the substantial computational requirements imposed by end-to-end LLM fine-tuning. This persistent challenge motivates our proposed approach.

\subsection{Motivation for Developmental Federated Tuning}
To address the persistent challenge of substantial computational burdens, and unlike existing works that update the LLM in an end-to-end manner throughout the federated process, which can be prohibitively resource-intensive, we propose a different paradigm. 
Drawing inspiration from human cognitive development~\cite{bengio2009curriculum,sweller2008human,mcardle2014human}, where learning progresses incrementally rather than instantaneously, we aim to mitigate these computational burdens by progressively cultivating a more capable LLM from a compact foundation. 

Specifically, we decompose the fine-tuning process into $S$ stages, mimicking different periods in human learning. 
The submodel capacity (i.e., the number of layers) at each stage is denoted as $\{L_{1}, L_{2}, \dots, L_{S}\}$, forming a strictly monotonically increasing sequence where $L_{s_{1}} < L_{s_{2}}$ for any $s_{1} < s_{2}$. The final stage capacity $L_{S}$ equals $L$, encompassing all layers of the LLM. Additionally, the knowledge acquired in each stage seamlessly transfers to the submodel of the subsequent stage, providing optimized initialization parameters. Compared to end-to-end fine-tuning, this developmental paradigm significantly reduces resource overhead for edge devices while achieving superior performance through a smoother optimization trajectory.
In this way, \our enables participating devices to efficiently fine-tune an $L$-layer LLM for downstream tasks. 

\section{Developmental Federated Tuning (\our)}

\begin{figure}[!t]
  \centering
  \includegraphics[width=1\linewidth]{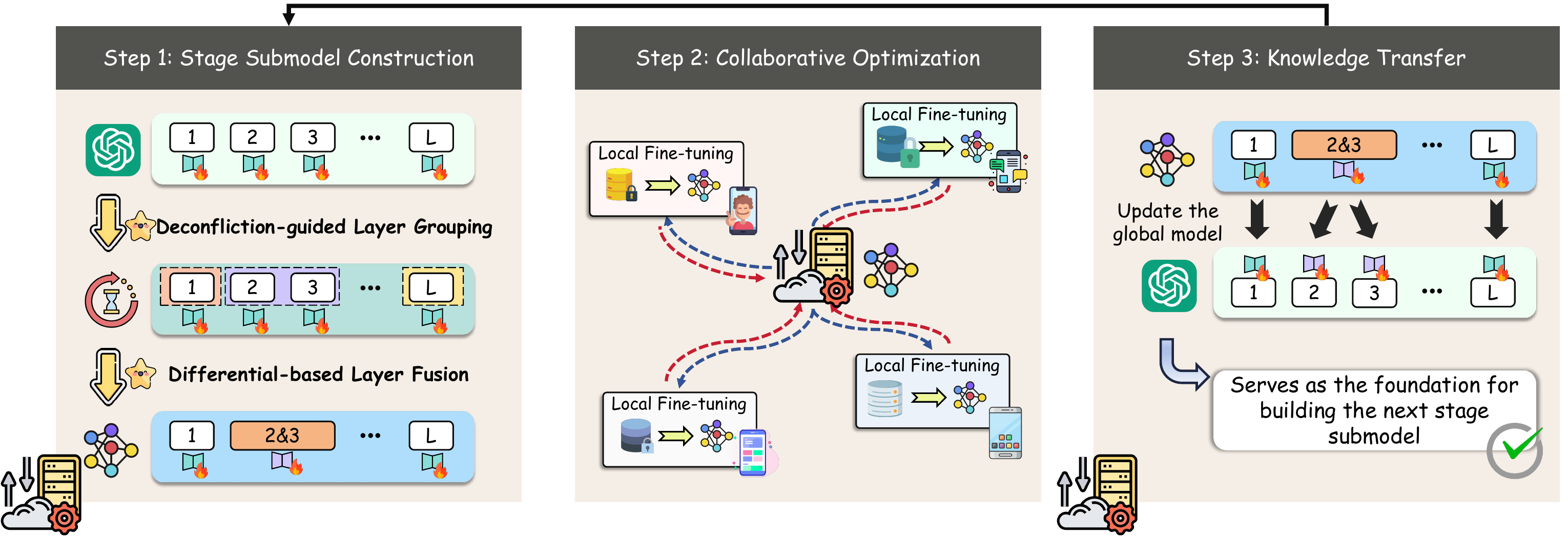}
    \caption{Overview of \our: The server first constructs the stage-specific submodel (step \ding{172}), followed by collaborative optimization across edge devices (step \ding{173}). After each stage, the acquired knowledge is employed to update the global model, which serves as the foundation for building the subsequent stage submodel (step \ding{174}).}
    \vspace{-12pt}
\label{workflow}
\end{figure}

\subsection{Overview}

Figure~\ref{workflow} depicts the overview of \our, which consists of three main steps. Initially, the server constructs a stage-specific submodel (step \ding{172}). Subsequently, the federated fine-tuning process commences, where participating devices collaboratively fine-tune the submodel (step \ding{173}). Upon completion of the current stage, the process advances to the next stage, where the acquired knowledge is used to update the global model (step \ding{174}) and is seamlessly transferred
to the subsequent stage submodel (Section~\ref{sec_Knowledge_inheritance}). This progressive model training process continues until the completion of the $S$-th stage.
During the stage submodel construction process, the server first employs the deconfliction-guided layer grouping mechanism (Section~\ref{sec_DGLF}) to cluster layers with minimal parameter conflicts into the same group. Then, it applies the differential-based layer fusion strategy (Section~\ref{sec_LAIF}) to integrate intra-group information, generating a representative layer for each group. After that, these representative layers are concatenated sequentially to construct the stage-specific submodel.

\subsection{Deconfliction-guided Layer Grouping}\label{sec_DGLF}


\begin{figure}[!t]
  \centering
  \includegraphics[width=0.92\linewidth]{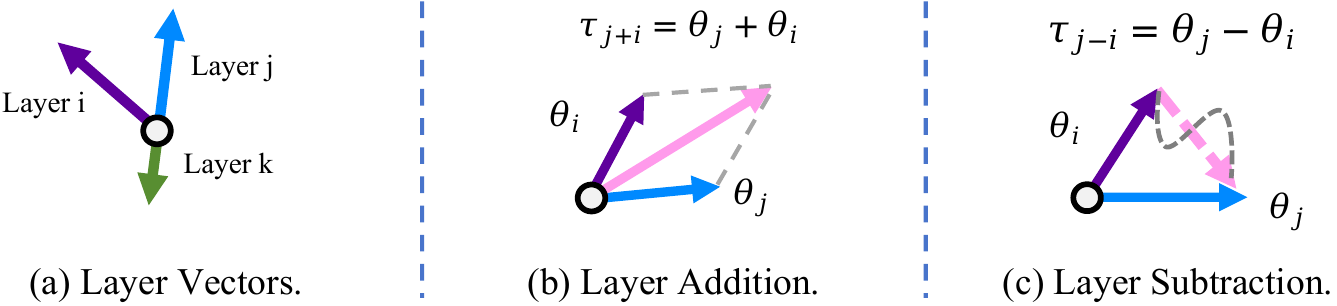}
  \caption{An illustration of layer vectors and layer arithmetic operations.
  }
\vspace{-12pt}
\label{fig_vector}
\end{figure}

As shown in Figure~\ref{fig_vector}(a), parameters of each layer can be represented as corresponding layer vectors, with varying degrees of parameter conflicts between different layers. 
When constructing representative layers for each group, significant parameter conflicts between layers can lead to substantial information loss, as parameters with opposing signs may neutralize each other's unique contributions during the layer fusion process.
To ensure the effectiveness of layer fusion, we propose a deconfliction-guided layer grouping (DGLG) mechanism that clusters layers with minimal parameter conflicts into the same group to preserve their respective knowledge. Specifically, the server initially calculates the inter-layer parameter similarity using Equation~\eqref{Eq_Similarity}:
\begin{equation}
    \label{Eq_Similarity}
    \text{sim}(\theta_i, \theta_j) = \frac{\langle \theta_i, \theta_j \rangle}{\|\theta_i\| \|\theta_j\|},
\end{equation}
where $\theta_i$ and $\theta_j$ denote the parameters of layers $i$ and $j$, respectively, including their corresponding LoRA parameters. This calculation generates a layer similarity matrix $\mathbf{W}$, where each element $w_{ij}$ represents the parameter similarity between layers $i$ and $j$.
Higher similarity values indicate lower parameter conflicts, suggesting these layers should be grouped together. Conversely, lower similarity values signify more severe parameter conflicts, necessitating the assignment of these layers to different groups.
Based on the similarity matrix $\mathbf{W}$, we construct a complete undirected graph 
$G = (\mathcal{V}, \mathcal{E})$, where 
$\mathcal{V} = \{v_1, v_2, ..., v_L\}$ represents the set of layers and 
$\mathcal{E} = \{\text{sim}(v_i, v_j) | v_i, v_j \in \mathcal{V}, w_{ij}=w_{ji}\}$ 
denotes the set of edges weighted by layer similarities.
The objective is to partition graph $G$ into $L_s$ non-overlapping groups $\{\mathrm{g}_{n}\}_{n=1}^{L_s}$ for stage $s$, which can be formally expressed as:
\begin{equation}
    \small
    \label{eq_partition_obj}
    \begin{aligned}
    \min_{\{\mathrm{g}_1,\mathrm{g}_2,...,\mathrm{g}_{L_s}\}} &\sum_{n=1}^{L_s} \sum_{m \neq n} \text{cut}(\mathrm{g}_n, \mathrm{g}_m), \text{where}~     \text{cut}(\mathrm{g}_n, \mathrm{g}_m) = \sum_{p \in \mathrm{g}_n} \sum_{q \in \mathrm{g}_m} w_{pq}, \\
    \text{s.t.}~ &\forall m,n \in \{1,2,...,L_s\}, m \neq n \Rightarrow \mathrm{g}_m \cap \mathrm{g}_n = \emptyset~\text{and}~ \bigcup_{n=1}^{L_s} \mathrm{g}_n = \mathcal{V}.
    \end{aligned}
\end{equation}
To solve the optimization problem in Equation~\eqref{eq_partition_obj}, we first construct the degree matrix $\mathbf{D} = \text{diag}(d_1,\ldots,d_L)$, where $d_i = \sum_{j=1}^L w_{ij}$ represents the sum of weights connected to vertex $v_{i}$. We then compute the Laplacian matrix as $\mathbf{L} = \mathbf{D} - \mathbf{W}$ and perform eigenvalue decomposition on $\mathbf{L}$ to obtain the eigenvectors corresponding to the $L_s$ smallest eigenvalues. These eigenvectors are stacked column by column to form the embedding matrix $\mathbf{E} \in \mathbb{R}^{L \times L_s}$. 
Finally, k-means clustering is applied to $\mathbf{E}$ to partition the vertex set $\mathcal{V}$ into $L_s$ disjoint groups. 
This process can be formally expressed as:
\begin{equation}
\small
\begin{aligned}
    \{\mathrm{g}_1, \ldots, \mathrm{g}_{L_s}\} &= \text{k-means}\left(\mathbf{E}, L_s \right), \quad \mathbf{E} = [\mathbf{v}_1, \ldots, \mathbf{v}_{L_{s}}], \\
    \text{where} \quad \mathbf{L} &= \mathbf{D} - \mathbf{W}, \quad \mathbf{D} = \text{diag}\left(\sum_{j=1}^L w_{1j}, \ldots, \sum_{j=1}^L w_{Lj}\right), \\
    \mathbf{L} \mathbf{v}_t &= \lambda_t \mathbf{v}_t, \quad \forall t \in \{1, \ldots, L_s\}, \quad \text{s.t.} \quad \lambda_1 \leq \lambda_2 \leq \cdots \leq \lambda_{L_s},
\end{aligned}
\end{equation}
where $\lambda_{t}$ and $\mathbf{v}_{t}$ represent the $t$-th eigenvalues and corresponding eigenvectors of $\mathbf{L}$.
Through this deconfliction-guided layer grouping mechanism, we can partition the $L$ layers of the global model into $L_s$ groups $\{\mathrm{g}_{n}\}_{n=1}^{L_s}$, where layers within each group exhibit minimal parameter conflicts.

\subsection{Differential-based Layer Fusion}\label{sec_LAIF}

After obtaining the partitioned groups, we proceed to construct a representative layer for each group. To effectively synthesize these representative layers, we introduce the differential-based layer fusion (DBLF) strategy, which consolidates layer information within each group through well-defined arithmetic operations. 
As illustrated in Figure~\ref{fig_vector}(b), the layer addition operation merges knowledge from two distinct layers, yielding a composite layer that encapsulates the semantic information of both source components.
Figure~\ref{fig_vector}(c) illustrates the layer subtraction operation, which distills the unique semantic information present in one layer relative to another. For any given layers $i$ and $j$, these operations are defined as follows:
\begin{equation}
\begin{aligned}
\tau_{j+i} &= \theta_j + \theta_i, \\
\tau_{j-i} &= \theta_j - \theta_i,
\end{aligned}
\end{equation}
where $\tau_{j+i}$ and $\tau_{j-i}$ denote the resulting parameter vectors after addition and subtraction operations, respectively. 
These operations enable precise knowledge editing in the parametric space.
A naive approach for intra-group information integration involves performing the addition operation on all layers. However, this method introduces significant information redundancy, as layers within the same group  $\mathrm{g}_{n}$ typically share similar functional characteristics.
This redundancy limits the submodel's capability to capture diverse and meaningful representations.

To address this challenge, instead of indiscriminately merging all information, DBLF selectively integrates the unique semantic information of each layer.
Specifically, it designates the first layer of each group as the \emph{anchor layer} and computes the information differentials of other layers relative to this \emph{anchor layer} through the layer subtraction operation. During layer fusion, only these information differentials are encapsulated into the \emph{anchor layer}, thereby effectively preserving each layer's essential information while eliminating redundancy. This fusion process can be formulated as:
\begin{equation}
\vartheta^{\mathrm{g_{n}}} = \theta_{\text{\emph{anchor}}} + \beta \sum_{j \in \mathrm{g_{n}}} (\theta_j - \theta_{\text{\emph{anchor}}}),
\end{equation}
where $\beta$ is a weighting factor, and $\vartheta^{\mathrm{g_{n}}}$ denotes the representative layer of group $\mathrm{g_{n}}$, which encapsulates the distinctive features from all constituent layers in the group.
These derived representative layers are then concatenated sequentially to construct a stage-specific submodel for federated fine-tuning.

\subsection{Knowledge Transfer}\label{sec_Knowledge_inheritance}

Cross-stage knowledge transfer plays a crucial role in cultivating high-performance LLMs, analogous to the human cognitive process where knowledge structures are progressively built upon established foundations. The knowledge acquired in each stage provides optimized initialization parameters for the subsequent stage's submodel, thereby accelerating convergence and enhancing overall model performance.
Through strategic layer clustering and representative layer construction, the encoded knowledge in $\{\vartheta^{\mathrm{g_{n}}}\}_{n=1}^{L_s}$ can be directly utilized to update all layers within their respective groups $\{\mathrm{g_{n}}\}_{n=1}^{L_s}$, as shown in Figure~\ref{workflow}.
The rationale lies in that functionally similar layers inherently exhibit similar parameter distributions and learning patterns. Notably, we only update the LoRA parameters of each layer. This knowledge transfer process generates an updated global model, which serves as the foundation for constructing the next-stage submodel, thereby ensuring seamless knowledge transfer across stages. Moreover, we provide the convergence analysis of \our in Appendix~\ref{theoretical_appendix}.

\section{Experiments}

\subsection{Experimental Setup}\label{sec_setup}

Following  OpenFedLLM~\cite{ye2024openfedllm}, we evaluate the effectiveness of \textsc{\our} on three LLaMA-based models with different parameter scales: LLaMA2-7B~\cite{llama2}, LLaMA3.1-8B~\cite{grattafiori2024llama}, and LLaMA2-13B~\cite{llama2}. Additionally, we fine-tune these models using the Alpaca-GPT4~\cite{peng2023instruction} dataset and evaluate the performance of the federated fine-tuned models on both close-ended and open-ended benchmarks. Specifically, the close-ended benchmarks include TruthfulQA~\cite{lin2022truthfulqa}, MMLU~\cite{hendrycks2020measuring}, IFEval~\cite{zhou2023instructionfollowing}, and BBH~\cite{bbh}, which assess the models' capabilities in honesty and truthfulness, knowledge breadth, instruction following, and reasoning, respectively. The open-ended benchmarks, including Vicuna-Bench~\cite{chiang2023vicuna} and MT-Bench~\cite{zheng2024judging}, evaluate the models' performance in multi-turn dialogue scenarios. 

The fine-tuning process is divided into four stages ($S=4$) for all models, with each stage's submodel receiving an equal number of federated fine-tuning rounds. The capacity of the submodels doubles at each stage. Specifically, for LLaMA2-7B and LLaMA3.1-8B, the submodel capacities across the four stages are \{4, 8, 16, 32\}, whereas for LLaMA2-13B, they are \{5, 10, 20, 40\}. We set the hyperparameter $\beta$ to 0.1 for LLaMA2-7B and LLaMA3.1-8B, and 0.15 for LLaMA2-13B.
Additional implementation details are provided in Appendix~\ref{appendix_implement}.



\definecolor{steelbluev2}{HTML}{DAE8FC}
\definecolor{steelblue}{HTML}{82B0D2}
\definecolor{color3}{HTML}{FEFAE0}
\definecolor{my_c1}{HTML}{d1e9d2}

\begin{table*}[!t]
\caption{
\textbf{Performance evaluation of \textsc{\textbf{\our}} against baseline methods on instruction tuning tasks}.
\textbf{Bold} and \underline{underlined} values denote the best and second-best results, respectively. }
\label{tab_overall_performance}
\renewcommand{\arraystretch}{0.9} 
  \centering
  \resizebox{1\textwidth}{!}{
    \begin{tabular}{lccccccccccc}
    \toprule[1.5pt]
    \multirow{2}{*}{\textbf{Method}}&\multicolumn{5}{c}{\textbf{Close-Ended Benchmark $\uparrow$}} &
    \multicolumn{4}{c}{\textbf{Open-Ended Benchmark $\uparrow$}}\\
    \cmidrule(lr{3pt}){2-6} \cmidrule(lr{3pt}){7-10}

    & TruthfulQA & MMLU & IFEval & BBH  & \textbf{Average} & Vicuna & MT-1 & MT-2&\textbf{Average} \\
   
    \midrule[1.5pt]
    &\multicolumn{9}{c}{\textbf{LLaMA2-7B (INT4)}~\cite{llama2}} \\ 
    \midrule[1.5pt]
    FedIT&47.57 & 42.45 & 31.76 & 39.28 & 40.27 & 8.18 & 4.77 & 1.98 & 4.98 \\
    DoFIT      & 48.32 & 43.04 & 32.62 & 39.59 & 40.89 & 8.19 & 4.92 & 2.13 & 5.08\\
    C2A        &46.71 & 41.83 & 29.45 & 36.07 & 38.52 & 7.66 & 3.97 & 1.88 & 4.50 \\
    \midrule
    ProgFed    &\underline{48.60} & \underline{43.14} & 32.54 & \underline{39.73} & \underline{41.00} & 8.20 & 4.88 & 2.19 & 5.09\\
    FLoRA & 47.76 & 42.64 & 32.08 & 39.25 & 40.43 & 8.21 & 4.85 & 2.02 & 5.03\\
    FedSA-LoRA & 48.24 & 42.91 & \underline{32.71} & 39.36 & 40.81 & \underline{8.26} & \underline{5.09} & \underline{2.31} & \underline{5.22}\\
    \rowcolor{my_c1!50}
    \textsc{\our} & \textbf{50.28} & \textbf{44.15} & \textbf{33.97} & \textbf{40.93} & \textbf{42.33} & \textbf{8.41} & \textbf{5.76} & \textbf{2.92} & \textbf{5.70}\\

    \midrule[1.5pt]

    &\multicolumn{9}{c}{\textbf{LLaMA3.1-8B (INT4)}~\cite{grattafiori2024llama}} \\ 
    \midrule[1.5pt]
    FedIT & 48.07 & 63.31 & 47.32 & 62.69 & 55.35 & 8.89 & 6.54 & 5.03 & 6.82 \\
    DoFIT & 49.12 & 65.17 & 51.66 & 65.21 & 57.79 & 9.01 & 6.72 & 5.22 & 6.98\\
    C2A & 48.99 & 63.76 & 46.10 & 61.85 & 55.18 & 8.74 & 6.67 & 4.98 & 6.80\\
    \midrule
    
    ProgFed & 53.12 & 66.77 & 54.55 & 66.03 & 60.12 & \underline{9.07} & 6.85 & 5.08 & 7.00\\
    FLoRA & 50.23 & 64.95 & 50.47 & 64.93 & 57.65 & 8.96 & 6.75 & 5.11 & 6.94 \\
    FedSA-LoRA & \underline{53.29} & \underline{66.87} & \underline{56.17} & \underline{67.56} & \underline{60.97} & 9.03 & \underline{6.92} & \underline{5.41} & \underline{7.12}  \\
    \rowcolor{my_c1!50}
    \textsc{\our} & \textbf{55.23} & \textbf{68.42} & \textbf{62.29} & \textbf{71.04} & \textbf{64.25} & \textbf{9.18} & \textbf{7.63} & \textbf{6.57} & \textbf{7.79} \\

    \midrule[1.5pt]

    &\multicolumn{9}{c}{\textbf{LLaMA2-13B (INT4)}~\cite{llama2}} \\ 
    \midrule[1.5pt]
    FedIT &52.40 & 55.45 & 40.33 & 46.14 & 48.58 & 8.37 & 5.17 & 3.01 & 5.52 \\
    DoFIT & 54.77 & 56.09 & 41.68 & 46.41 & 49.74 & 8.37 & 5.19 & 3.34 & 5.63\\
    C2A   & 53.91 & 54.33 & 38.96 & 45.06 & 48.07 & 8.05 & 5.08 & 3.26 & 5.46\\
    \midrule

    ProgFed & 55.01 & 57.38 & 42.13 & 46.36 & 50.22 & 8.38 & 5.28 & 3.07 & 5.58 \\
    FLoRA & 54.26 & 56.23 & 41.49 & 46.32 & 49.58 & 8.40 & 5.22 & 3.15 & 5.59\\
    FedSA-LoRA & \underline{55.73} & \underline{57.51} & \underline{43.21} & \underline{46.91} & \underline{50.84} & \underline{8.49} & \underline{5.39} & \underline{3.45} & \underline{5.78} \\
    \rowcolor{my_c1!50}
    \textsc{\our} & \textbf{57.19} & \textbf{58.74} & \textbf{46.45} & \textbf{48.70} & \textbf{52.77} & \textbf{8.67} & \textbf{6.18} & \textbf{4.52} & \textbf{6.46} \\
   
    \bottomrule[1.5pt]
    \end{tabular}
    }
    \vspace{-4mm}
\end{table*}

\subsection{Baselines}

\textbf{Resource-Unaware Methods}.  FedIT~\citep{zhang2024towards} directly integrates LoRA with FedAvg to enable federated instruction tuning across devices. DoFIT~\citep{xu2024dofit} is a domain-aware method employing specialized LoRA weight initialization and aggregation strategies to mitigate catastrophic forgetting across different domains. C2A~\citep{kim2023client} is a hypernetwork-based approach tackling data heterogeneity by dynamically generating device-specific adapters.

\noindent \textbf{Resource-Aware Methods}.  
ProgFed~\citep{wang2022progfed} partitions the global model into blocks and gradually incorporates them for training. FLoRA~\citep{wang2024flora} allocates different LoRA ranks to devices based on their computational resources. FedSA-LoRA~\citep{guo2024selective} identifies that matrix $\mathbf{A}$ in the LoRA module captures general features, and thus only shares matrices $\mathbf{A}$ with the server to reduce resource costs.

\subsection{Performance Evaluation}

Table~\ref{tab_overall_performance} provides a comprehensive performance comparison across various methods. The results demonstrate that \our consistently outperforms baseline methods across all experimental settings. 

1) \textbf{Comparison with Resource-Unaware Methods.} Resource-unaware methods uniformly demonstrate inferior performance. In close-ended benchmarks, FedIT shows significant average performance degradation of 2.06\%, 8.9\%, and 4.19\% compared to \our on LLaMA2-7B, LLaMA3.1-8B, and LLaMA2-13B, respectively. Similarly, for open-ended benchmarks, the method exhibits average performance gaps of 0.72, 0.97, and 0.94 across these models. This performance deterioration primarily stems from the noise introduced by FedIT's independent aggregation of matrices $\mathbf{A}$ and $\mathbf{B}$. While DoFIT achieves moderate improvements through its specialized initialization and aggregation strategies, it still demonstrates a substantial performance gap of up to 10.63\% compared to \our on LLaMA3.1-8B. Furthermore, C2A performs notably worse than \our, with average performance drops of up to 9.07\% and 1.2 in close-ended and open-ended benchmarks respectively, highlighting the advantages of LoRA over adapter-based approaches.

2) \textbf{Comparison with Resource-Aware Methods.} While resource-aware methods generally demonstrate superior performance compared to resource-unaware counterparts, they still exhibit notable performance gaps relative to \our. Specifically, ProgFed shows average performance degradation of 1.33\% and 0.61 on LLaMA2-7B, 4.13\% and 0.79 on LLaMA3.1-8B, and 2.55\% and 0.88 on LLaMA2-13B for close-ended and open-ended benchmarks respectively. FedSA-LoRA exhibits similar performance degradation patterns to ProgFed, while FLoRA demonstrates more significant performance deterioration. In particular, for close-ended benchmarks, FedSA-LoRA shows average performance decrements ranging from 1.52\% to 3.28\% across these models, whereas FLoRA exhibits more substantial degradation, with decrements spanning from 1.9\% to 6.6\%.
The superior performance of \our stems from its developmental paradigm, which progressively builds a powerful LLM from a compact foundation, effectively preventing convergence to local minima.

\subsection{Efficiency Evaluation}

\definecolor{red1}{RGB}{203,104,104}
\definecolor{blue1}{RGB}{104,155,203}
\definecolor{red2}{HTML}{ffb3a7}
\definecolor{blue2}{RGB}{108,179,211}
\definecolor{uorange}{RGB}{247,175,89}
\definecolor{upurple}{RGB}{148,137,250}
\definecolor{pink1}{HTML}{FFF0F6}
\definecolor{pink2}{HTML}{FCC2D7}
\definecolor{pink3}{HTML}{FAA2CA}

\definecolor{block}{HTML}{2200ff}

\definecolor{cyan1}{HTML}{BFEFFF}
\definecolor{cyan2}{HTML}{B2DFEE}
\definecolor{cyan3}{HTML}{9AC0CD}

\definecolor{lightsteelblue1}{HTML}{CAE1FF}
\definecolor{lightsteelblue2}{HTML}{BCD2EE}

\definecolor{color1}{HTML}{f7fbff}
\definecolor{color2}{HTML}{F7FCC9}
\definecolor{color3}{HTML}{D2EFB2}
\definecolor{color4}{HTML}{F0DDCC}

\definecolor{color5}{HTML}{FECEA3}
\definecolor{color6}{HTML}{FFEDED}
\definecolor{color7}{HTML}{a4e2c6}

\definecolor{my_purple}{HTML}{eedeb0}

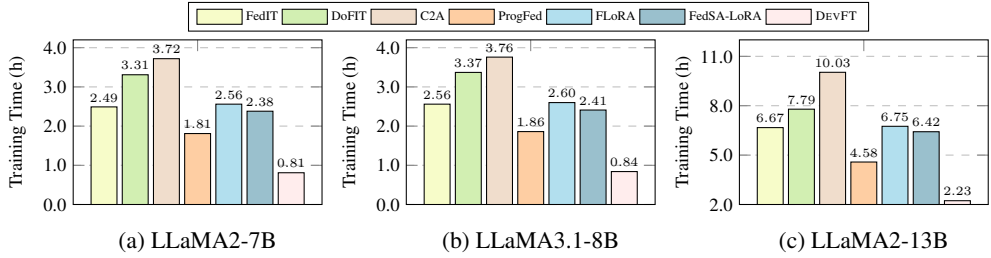
\begin{figure}[!t]
\centering
\begin{tikzpicture}
\scriptsize{

\begin{axis}[
    at={(0em,0em)},
    ymajorgrids,
    grid style=dashed,
    ylabel={\scriptsize{Training Time (h)}},
    legend style={at={(1.83,1.05)}, anchor=south, legend columns=-1, nodes={scale=0.7, transform shape}},
    ybar,
    enlarge x limits=0.3,
    xtick align=inside,
    height=.27\textwidth,
    width=.35\textwidth,
    bar width=1.4em,
    nodes near coords,
    nodes near coords align={vertical},
    nodes near coords style={font=\tiny, scale=0.8,/pgf/number format/fixed, /pgf/number format/precision=1},
    every node near coord/.append style={/pgf/number format/.cd, fixed, fixed zerofill, precision=2},
    xlabel={\footnotesize{(a) LLaMA2-7B}},
    symbolic x coords={0},
    xtick=data,
    ymin=0,
    ymax=4.2,
    ytick={0.0,1.0,2.0,3.0,4.0},
    yticklabels={0.0,1.0,2.0,3.0,4.0},
    xticklabels={},  
    ylabel style={yshift=-2.em},
    xlabel style={yshift=1em,align=center},
    yticklabel style={/pgf/number format/fixed},
]

    \addplot[fill=color2,draw=black, area legend] coordinates {(0, 2.49)};
    \addlegendentry{FedIT}
    \addplot[fill=color3, draw=black, area legend] coordinates {(0, 3.31)};
    \addlegendentry{DoFIT}
    \addplot[fill=color4,draw=black, area legend] coordinates {(0, 3.72)};
    \addlegendentry{C2A}
    \addplot[fill=color5,draw=black, area legend] coordinates {(0, 1.81)};
    \addlegendentry{ProgFed}

    \addplot[fill=cyan2,draw=black, area legend] coordinates {(0, 2.56)};
    \addlegendentry{FLoRA}

    \addplot[fill=cyan3,draw=black, area legend] coordinates {(0, 2.38)};
    \addlegendentry{FedSA-LoRA}

    \addplot[fill=color6,draw=black, area legend] coordinates {(0, 0.81)};
    \addlegendentry{\textsc{\our}}

\end{axis}

\begin{axis}[
    at={(18em,0em)},
    ymajorgrids,
    grid style=dashed,
    ylabel={\scriptsize{Training Time (h)}},
    legend style={at={(1.13,1.05)}, anchor=south, legend columns=-1, nodes={scale=0.7, transform shape}},
    ybar,
    enlarge x limits=0.3,
    xtick align=inside,
    height=.27\textwidth,
    width=.35\textwidth,
    bar width=1.4em,
    nodes near coords,
    nodes near coords align={vertical},
    nodes near coords style={font=\tiny, scale=0.8,/pgf/number format/fixed, /pgf/number format/precision=1},
    every node near coord/.append style={/pgf/number format/.cd, fixed, fixed zerofill, precision=2},
    xlabel={\footnotesize{(b) LLaMA3.1-8B}},
    symbolic x coords={0},
    xtick=data,
    ymin=0,
    ymax=4.2,
    ytick={0.0,1.0,2.0,3.0,4.0},
    yticklabels={0.0,1.0,2.0,3.0,4.0},
    xticklabels={},  
    ylabel style={yshift=-2.em},
    xlabel style={yshift=1em,align=center},
    yticklabel style={/pgf/number format/fixed},
]

    \addplot[fill=color2,draw=black, area legend] coordinates {(0, 2.56)};
    \addplot[fill=color3, draw=black, area legend] coordinates {(0, 3.37)};
    \addplot[fill=color4,draw=black, area legend] coordinates {(0, 3.76)};
    \addplot[fill=color5,draw=black, area legend] coordinates {(0, 1.86)};

    \addplot[fill=cyan2,draw=black, area legend] coordinates {(0, 2.60)};

    \addplot[fill=cyan3,draw=black, area legend] coordinates {(0, 2.41)};

    \addplot[fill=color6,draw=black, area legend] coordinates {(0, 0.84)};

\end{axis}

\begin{axis}[
    at={(36em,0em)},
    ymajorgrids,
    grid style=dashed,
    ylabel={\scriptsize{Training Time (h)}},
    legend style={at={(1.13,1.05)}, anchor=south, legend columns=-1, nodes={scale=0.7, transform shape}},
    ybar,
    enlarge x limits=0.3,
    xtick align=inside,
    height=.27\textwidth,
    width=.35\textwidth,
    bar width=1.4em,
    nodes near coords,
    nodes near coords align={vertical},
    nodes near coords style={font=\tiny, scale=0.8,/pgf/number format/fixed, /pgf/number format/precision=1},
    every node near coord/.append style={/pgf/number format/.cd, fixed, fixed zerofill, precision=2},
    xlabel={\footnotesize{(c) LLaMA2-13B}},
    symbolic x coords={0},
    xtick=data,
    ymin=2.0,
    ymax=12.0,
    ytick={2.0,5.0,8.0,11.0},
    yticklabels={2.0,5.0,8.0,11.0},
    xticklabels={},  
    ylabel style={yshift=-2.em},
    xlabel style={yshift=1em,align=center},
    yticklabel style={/pgf/number format/fixed},
]

    \addplot[fill=color2,draw=black, area legend] coordinates {(0, 6.67)};
    \addplot[fill=color3, draw=black, area legend] coordinates {(0, 7.79)};
    \addplot[fill=color4,draw=black, area legend] coordinates {(0, 10.03)};
    \addplot[fill=color5,draw=black, area legend] coordinates {(0, 4.58)};

    \addplot[fill=cyan2,draw=black, area legend] coordinates {(0, 6.75)};

    \addplot[fill=cyan3,draw=black, area legend] coordinates {(0, 6.42)};

    \addplot[fill=color6,draw=black, area legend] coordinates {(0, 2.23)};

\end{axis}

}
\end{tikzpicture}

\vspace{-1mm}
\caption{Comparative analysis of cumulative local \textit{training time} across different methods.}
\label{fig:overhead_Time}
\vspace{-3mm}
\end{figure}

\definecolor{red1}{RGB}{203,104,104}
\definecolor{blue1}{RGB}{104,155,203}
\definecolor{red2}{HTML}{ffb3a7}
\definecolor{blue2}{RGB}{108,179,211}
\definecolor{uorange}{RGB}{247,175,89}
\definecolor{upurple}{RGB}{148,137,250}
\definecolor{pink1}{HTML}{FFF0F6}
\definecolor{pink2}{HTML}{FCC2D7}
\definecolor{pink3}{HTML}{FAA2CA}

\definecolor{block}{HTML}{2200ff}

\definecolor{cyan1}{HTML}{BFEFFF}
\definecolor{cyan2}{HTML}{B2DFEE}
\definecolor{cyan3}{HTML}{9AC0CD}

\definecolor{lightsteelblue1}{HTML}{CAE1FF}
\definecolor{lightsteelblue2}{HTML}{BCD2EE}

\definecolor{color1}{HTML}{f7fbff}
\definecolor{color2}{HTML}{F7FCC9}
\definecolor{color3}{HTML}{D2EFB2}
\definecolor{color4}{HTML}{F0DDCC}

\definecolor{color5}{HTML}{FECEA3}
\definecolor{color6}{HTML}{FFEDED}
\definecolor{color7}{HTML}{a4e2c6}

\definecolor{my_purple}{HTML}{eedeb0}

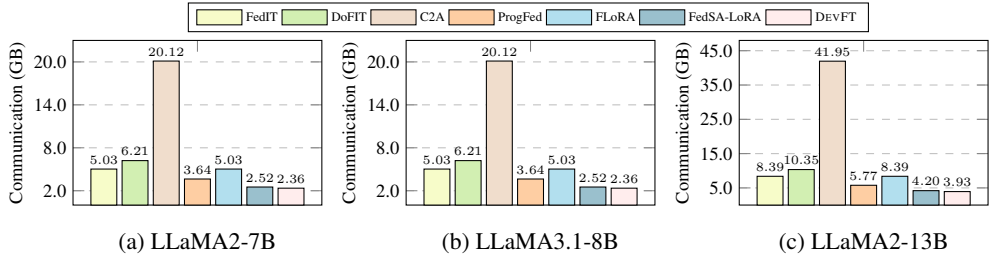
\begin{figure}[!t]
\centering
\begin{tikzpicture}
\scriptsize{

\begin{axis}[
    at={(0em,0em)},
    ymajorgrids,
    grid style=dashed,
    ylabel={\scriptsize{Communication (GB)}},
    legend style={at={(1.83,1.05)}, anchor=south, legend columns=-1, nodes={scale=0.7, transform shape}},
    ybar,
    enlarge x limits=0.3,
    xtick align=inside,
    height=.27\textwidth,
    width=.35\textwidth,
    bar width=1.4em,
    nodes near coords,
    nodes near coords align={vertical},
    nodes near coords style={font=\tiny, scale=0.8,/pgf/number format/fixed, /pgf/number format/precision=1},
    every node near coord/.append style={/pgf/number format/.cd, fixed, fixed zerofill, precision=2},
    xlabel={\footnotesize{(a) LLaMA2-7B}},
    symbolic x coords={0},
    xtick=data,
    ymin=0.0,
    ymax=23.0,
    ytick={2.0,8.0,14.0,20.0},
    yticklabels={2.0,8.0,14.0,20.0},
    xticklabels={},  
    ylabel style={yshift=-2.em},
    xlabel style={yshift=1em,align=center},
    yticklabel style={/pgf/number format/fixed},
]

    \addplot[fill=color2,draw=black, area legend] coordinates {(0, 5.03)};
    \addlegendentry{FedIT}
    \addplot[fill=color3, draw=black, area legend] coordinates {(0, 6.21)};
    \addlegendentry{DoFIT}
    \addplot[fill=color4,draw=black, area legend] coordinates {(0, 20.12)};
    \addlegendentry{C2A}
    \addplot[fill=color5,draw=black, area legend] coordinates {(0, 3.64)};
    \addlegendentry{ProgFed}

    \addplot[fill=cyan2,draw=black, area legend] coordinates {(0, 5.03)};
    \addlegendentry{FLoRA}

    \addplot[fill=cyan3,draw=black, area legend] coordinates {(0, 2.52)};
    \addlegendentry{FedSA-LoRA}

    \addplot[fill=color6,draw=black, area legend] coordinates {(0, 2.36)};
    \addlegendentry{\textsc{\our}}

\end{axis}
\begin{axis}[
    at={(18em,0em)},
    ymajorgrids,
    grid style=dashed,
    ylabel={\scriptsize{Communication (GB)}},
    legend style={at={(1.83,1.05)}, anchor=south, legend columns=-1, nodes={scale=0.7, transform shape}},
    ybar,
    enlarge x limits=0.3,
    xtick align=inside,
    height=.27\textwidth,
    width=.35\textwidth,
    bar width=1.4em,
    nodes near coords,
    nodes near coords align={vertical},
    nodes near coords style={font=\tiny, scale=0.8,/pgf/number format/fixed, /pgf/number format/precision=1},
    every node near coord/.append style={/pgf/number format/.cd, fixed, fixed zerofill, precision=2},
    xlabel={\footnotesize{(b) LLaMA3.1-8B}},
    symbolic x coords={0},
    xtick=data,
    ymin=0.0,
    ymax=23.0,
    ytick={2.0,8.0,14.0,20.0},
    yticklabels={2.0,8.0,14.0,20.0},
    xticklabels={},  
    ylabel style={yshift=-2.em},
    xlabel style={yshift=1em,align=center},
    yticklabel style={/pgf/number format/fixed},
]

    \addplot[fill=color2,draw=black, area legend] coordinates {(0, 5.03)};
    \addplot[fill=color3, draw=black, area legend] coordinates {(0, 6.21)};
    \addplot[fill=color4,draw=black, area legend] coordinates {(0, 20.12)};
    \addplot[fill=color5,draw=black, area legend] coordinates {(0, 3.64)};

    \addplot[fill=cyan2,draw=black, area legend] coordinates {(0, 5.03)};

    \addplot[fill=cyan3,draw=black, area legend] coordinates {(0, 2.52)};

    \addplot[fill=color6,draw=black, area legend] coordinates {(0, 2.36)};

\end{axis}

\begin{axis}[
    at={(36em,0em)},
    ymajorgrids,
    grid style=dashed,
    ylabel={\scriptsize{Communication (GB)}},
    legend style={at={(1.13,1.05)}, anchor=south, legend columns=-1, nodes={scale=0.7, transform shape}},
    ybar,
    enlarge x limits=0.3,
    xtick align=inside,
    height=.27\textwidth,
    width=.35\textwidth,
    bar width=1.4em,
    nodes near coords,
    nodes near coords align={vertical},
    nodes near coords style={font=\tiny, scale=0.8,/pgf/number format/fixed, /pgf/number format/precision=1},
    every node near coord/.append style={/pgf/number format/.cd, fixed, fixed zerofill, precision=2},
    xlabel={\footnotesize{(c) LLaMA2-13B}},
    symbolic x coords={0},
    xtick=data,
    ymin=0.0,
    ymax=48.0,
    ytick={5.0,15.0,25.0,35.0,45.0},
    yticklabels={5.0,15.0,25.0,35.0,45.0},
    xticklabels={},  
    ylabel style={yshift=-2.em},
    xlabel style={yshift=1em,align=center},
    yticklabel style={/pgf/number format/fixed},
]

    \addplot[fill=color2,draw=black, area legend] coordinates {(0, 8.39)};
    \addplot[fill=color3, draw=black, area legend] coordinates {(0, 10.35)};
    \addplot[fill=color4,draw=black, area legend] coordinates {(0, 41.95)};
    \addplot[fill=color5,draw=black, area legend] coordinates {(0, 5.77)};

    \addplot[fill=cyan2,draw=black, area legend] coordinates {(0, 8.39)};

    \addplot[fill=cyan3,draw=black, area legend] coordinates {(0, 4.20)};

    \addplot[fill=color6,draw=black, area legend] coordinates {(0, 3.93)};

\end{axis}

}
\end{tikzpicture}

\vspace{-2mm}
\caption{Comparative analysis of total \textit{communication overhead} across different methods.}
\label{fig:overhead_communication}
\vspace{-6mm}
\end{figure}

In this section, we evaluate the efficiency of \our from both computation and communication perspectives.
Furthermore, we present a detailed analysis of training overhead across different stages to understand how \our effectively optimizes resource utilization.

\noindent \textbf{Computation Efficiency.} Instead of using floating-point operations per second (FLOPs) to evaluate computation efficiency, we employ wall-clock training time to provide a more intuitive reflection of real-world deployment efficiency for each method. Specifically, we measure the cumulative local training time required for each method to achieve convergence, with results shown in Figure~\ref{fig:overhead_Time}. Our experimental results demonstrate that \our significantly accelerates model convergence across all model architectures. Notably, for LLaMA2-7B, \our achieves up to 4.59$\times$ speedup in convergence time. This improvement can be attributed to the progressive training strategy of \our, where initial fine-tuning of smaller submodels significantly reduces computational overhead, while  knowledge transfer to larger submodels further expedites convergence.

\noindent \textbf{Communication Efficiency.} Figure~\ref{fig:overhead_communication} illustrates the total communication overhead required for each method to reach convergence. \our consistently achieves convergence with minimal communication costs across all experimental settings, reducing communication overhead by up to 10.67$\times$ on LLaMA2-13B. This communication efficiency stems from the fact that \our only transmits a small number of LoRA parameters to the server during the initial $S-1$ stages.

\noindent \textbf{Detailed Overhead Analysis.} To gain a deeper understanding of \our's efficiency, Figure~\ref{fig:per_round_overhead} illustrates the per-round resource consumption on each device for FedIT and \textsc{\our}, including training time, communication overhead, and memory usage.
FedIT consistently exhibits high resource demands throughout the fine-tuning process. In contrast, \textsc{\our} demonstrates a more efficient pattern, with resource requirements gradually increasing as the submodel capacity expands, thereby substantially reducing the training overhead.
In the early stages, particularly during the first stage, \textsc{\our} achieves significant resource savings compared to FedIT, reducing the per-round training time by 10.3$\times$, communication overhead by 4$\times$, and memory usage by 4$\times$.
Intriguingly, we discover that fine-tuning the reconstructed models of \textsc{\our} at each stage also yields acceleration compared to directly fine-tuning pre-trained models via API calls. For example, even in the fourth stage where the submodel grows to match the target model size, \textsc{\our} still achieves a 1.44$\times$ speedup per round.

\definecolor{Maroon}{HTML}{AE3135}
\definecolor{BLUE}{HTML}{6466AE}
\definecolor{my-green}{HTML}{8ECFC9}
\definecolor{my-yellow}{HTML}{FFBE7A}
\definecolor{my-blue}{HTML}{82B0D2}
\definecolor{my_c1}{HTML}{f1c550}
\definecolor{my_c2}{HTML}{1687a7}

\begin{figure}[!t]
\centering
\pgfplotsset{width=0.32\linewidth,height=0.25\linewidth,compat=1.15}
\footnotesize
\begin{tikzpicture}
\scriptsize{
\begin{axis}[
	at={(0em,0em)},
    xlabel={Round},
    ylabel={\tiny{Training Time (s)}},
    xmin=0.05, xmax=0.55,
    ymin=0, ymax=16,
    xtick={0.1, 0.2, 0.3, 0.4, 0.5},
    ytick={0, 5, 10, 15},
    ymajorgrids=true,
    xmajorgrids=true,
    grid style=dashed,
    xticklabels={0, 75, 150, 225, 300},
    x label style={at={(axis description cs:0.5,-0.15)},anchor=north},
    y label style={at={(axis description cs:-0.15,0.5)},anchor=south},
    legend style={
    	at={(0.35,0.45)},
    	anchor=south,
    	legend columns=1,
    	nodes={scale=0.8, transform shape}}
]
\addplot[
    color=my_c1,
    mark=square*,
    mark size=1.5pt,thick,line width=1.5pt,
    mark options={fill=my_c1,draw=my_c1,line width=1.8pt}
    ]
    coordinates {
    (0.1, 14.88)
    (0.2, 14.88)
    (0.3, 14.88)
    (0.4, 14.88)
    (0.5, 14.88)
    };
    \addlegendentry{FedIT}
\addplot[
    color=my_c2,
    mark=pentagon*,
    mark size=1.5pt,thick,line width=1.5pt,
    mark options={fill=my_c2,draw=my_c2,line width=1.8pt}
    ]
    coordinates {
    (0.1, 2.88/2)
    (0.2, 2.88/2)
    (0.2, 5.232/2)
    (0.3, 5.232/2)
    (0.3, 10.224/2)
    (0.4, 10.224/2)
    (0.4, 20.688/2)
    (0.5, 20.688/2)
    };
    \addlegendentry{\textsc{\our}}
    
\end{axis}

\begin{axis}[
	at={(16em,0em)},
    xlabel={Round},
    ylabel={\tiny{Communication (MB)}},
    xmin=0.05, xmax=0.55,
    ymin=0, ymax=9,
    xtick={0.1, 0.2, 0.3, 0.4, 0.5},
    ytick={0, 4, 8},
    ymajorgrids=true,
    xmajorgrids=true,
    grid style=dashed,
    xticklabels={0, 75, 150, 225 , 300},
    x label style={at={(axis description cs:0.5,-0.15)},anchor=north},
    y label style={at={(axis description cs:-0.15,0.5)},anchor=south},
    legend style={
    	at={(0.33,0.52)},
    	anchor=south,
    	legend columns=1,
    	nodes={scale=0.7, transform shape}}
]
\addplot[
    color=my_c1,
    mark=square*,
    mark size=1.5pt,thick,line width=1.5pt,
    mark options={fill=my_c1,draw=my_c1,line width=1.8pt}
    ]
    coordinates {
    (0.1, 8.6016)
    (0.2, 8.6016)
    (0.3, 8.6016)
    (0.4, 8.6016)
    (0.5, 8.6016)
    };
    \addlegendentry{FedIT}

\addplot[
    color=my_c2,
    mark=pentagon*,
    mark size=1.5pt,thick,line width=1.5pt,
    mark options={fill=my_c2,draw=my_c2,line width=1.8pt}
    ]
    coordinates {
    (0.1, 8.6016/8)
    (0.2, 8.6016/8)
    (0.2, 8.6016/4)
    (0.3, 8.6016/4)
    (0.3, 8.6016/2)
    (0.4, 8.6016/2)
    (0.4, 8.6016)
    (0.5, 8.6016)
    };
    \addlegendentry{\textsc{\our}}
\end{axis}

\begin{axis}[
	at={(32em,0em)},
    xlabel={Round},
    ylabel={\tiny{Memory Usage (GB)}},
    xmin=0.05, xmax=0.55,
    ymin=0, ymax=9,
    xtick={0.1, 0.2, 0.3, 0.4, 0.5},
    ytick={0, 4, 8},
    ymajorgrids=true,
    xmajorgrids=true,
    grid style=dashed,
    xticklabels={0, 75, 150, 225 , 300},
    x label style={at={(axis description cs:0.5,-0.15)},anchor=north},
    y label style={at={(axis description cs:-0.15,0.5)},anchor=south},
    legend style={
    	at={(0.33,0.52)},
    	anchor=south,
    	legend columns=1,
    	nodes={scale=0.7, transform shape}}
]
\addplot[
    color=my_c1,
    mark=square*,
    mark size=1.5pt,thick,line width=1.5pt,
    mark options={fill=my_c1,draw=my_c1,line width=1.8pt}
    ]
    coordinates {
    (0.1, 8.188)
    (0.2, 8.188)
    (0.3, 8.188)
    (0.4, 8.188)
    (0.5, 8.188)
    };
    \addlegendentry{FedIT}

\addplot[
    color=my_c2,
    mark=pentagon*,
    mark size=1.5pt,thick,line width=1.5pt,
    mark options={fill=my_c2,draw=my_c2,line width=1.8pt}
    ]
    coordinates {
    (0.1, 8.188/8)
    (0.2, 8.188/8)
    (0.2, 8.188/4)
    (0.3, 8.188/4)
    (0.3, 8.188/2)
    (0.4, 8.188/2)
    (0.4, 8.188)
    (0.5, 8.188)
    };
    \addlegendentry{\textsc{\our}}

\end{axis}
}
\end{tikzpicture}
\vspace{-2mm}
\caption{Resource consumption analysis of a device per round: \textit{training time, communication overhead, and memory usage} for FedIT and \textsc{\our}. The global model is LLaMA2-7B.}
\vspace{-5mm}
\label{fig:per_round_overhead}
\end{figure}
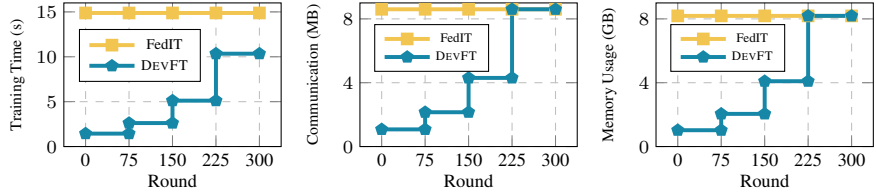

\begin{figure}[!h]
\vspace{-7mm}

    \centering
    \begin{minipage}[c]{0.48\textwidth}
    \centering
    \begin{table}[H]
    \caption{Ablation study on different layer grouping strategies.}
    \label{tab:ablation_technique1}
    \resizebox{\linewidth}{!}{
    \begin{tabular}{lccccccccc}
    \toprule[1.5pt]
    \multirow{2}{*}{\textbf{Method}} & \multicolumn{5}{c}{\textbf{Close-Ended Benchmark $\uparrow$}} \\
    \cmidrule(lr{3pt}){2-6}
    & TruthfulQA & MMLU & IFEval & BBH  & \textbf{Average} \\
    \midrule[1.5pt]
    &\multicolumn{5}{c}{\textbf{LLaMA2-7B (INT4)}~\cite{llama2}} \\ 
    \midrule[1.5pt]
    \rowcolor{my_black!40}
    DGLG            & \textbf{50.28} & \textbf{44.15} & \textbf{33.97} & \textbf{40.93} & \textbf{42.33} \\
    \textsc{Random} & 47.89 & 42.09 & 29.18 & 38.45 & 39.90 (\textcolor{red}{$\downarrow$ 2.43}) \\
    \textsc{Even}   & 45.41 & 39.83 & 25.04 & 36.73 & 36.25 (\textcolor{red}{$\downarrow$ 6.08})\\
    
    \midrule[1.5pt]
    
    &\multicolumn{5}{c}{\textbf{LLaMA3.1-8B (INT4)}~\cite{llama2}} \\ 
    \midrule[1.5pt]
    \rowcolor{my_black!40}
    DGLG            & \textbf{55.23} & \textbf{68.42} & \textbf{62.29} & \textbf{71.04} & \textbf{64.25} \\
    \textsc{Random} & 51.02 & 66.74 & 54.89 & 70.11 & 60.69 (\textcolor{red}{$\downarrow$ 3.56}) \\
    \textsc{Even}   & 48.51 & 62.50 & 50.01 & 70.03 & 57.76 (\textcolor{red}{$\downarrow$ 6.49}) \\
    \bottomrule[1.5pt]
    \end{tabular}
    }
	\end{table}
    \end{minipage}
\hfill 
\begin{minipage}[c]{0.48\textwidth}
    \centering
    \begin{table}[H]
    \caption{Ablation study on different representative layer construction methods.}
    \label{tab:ablation_technique2}
    \resizebox{\linewidth}{!}{
    \begin{tabular}{lccccccccc}
    \toprule[1.5pt]
    \multirow{2}{*}{\textbf{Method}}&\multicolumn{5}{c}{\textbf{Close-Ended Benchmark $\uparrow$}} \\
    \cmidrule(lr{3pt}){2-6} \cmidrule(lr{3pt}){7-8}

    & TruthfulQA & MMLU & IFEval & BBH  & \textbf{Average} \\
   
    \midrule[1.5pt]
    &\multicolumn{5}{c}{\textbf{LLaMA2-7B (INT4)}~\cite{llama2}} \\ 
    \midrule[1.5pt]
    \rowcolor{my_black!40}
    DBLF            & \textbf{50.28} & \textbf{44.15} & \textbf{33.97} & \textbf{40.93} & \textbf{42.33} \\
    \textsc{R-One} & 46.75 & 40.13 & 26.38 & 37.62 & 37.72 (\textcolor{red}{$\downarrow$ 4.61}) \\
    \textsc{Sum}    & 48.15 & 42.91 & 30.69 & 39.84 & 40.90 (\textcolor{red}{$\downarrow$ 1.43})\\
    \midrule[1.5pt]
    &\multicolumn{5}{c}{\textbf{LLaMA3.1-8B (INT4)}~\cite{llama2}} \\ 
    \midrule[1.5pt]
    \rowcolor{my_black!40}
    DBLF            & \textbf{55.23} & \textbf{68.42} & \textbf{62.29} & \textbf{71.04} & \textbf{64.25} \\
    \textsc{R-One} & 47.51 & 57.33 & 50.21 & 58.09 & 53.29 (\textcolor{red}{$\downarrow$ 10.96})\\
    \textsc{Sum}    & 52.74 & 65.18 & 58.47 & 68.39 & 61.20 (\textcolor{red}{$\downarrow$ 3.05})\\
    \bottomrule[1.5pt]
    \end{tabular}
    }
    \end{table}
\end{minipage}
\vspace{-4mm}
\end{figure}

\subsection{Ablation Study}

\textbf{Effect of the Deconfliction-guided Layer Grouping Mechanism.} To understand the significance of the deconfliction-guided layer grouping (DGLG) mechanism, we compare it with two baseline variants: random grouping (denoted as \textsc{Random}) and even grouping (denoted as \textsc{Even}).
As shown in Table~\ref{tab:ablation_technique1}, DGLG consistently outperforms both baselines across all experimental settings.
Specifically, for LLaMA2-7B, compared to DGLG, \textsc{Random} and \textsc{Even} exhibit average performance degradation of 2.43\% and 6.08\%, respectively. Similar trends are observed on LLaMA3.1-8B, where \textsc{Random} and \textsc{Even} result in average performance drops of 3.56\% and 6.49\%, respectively. These results demonstrate that the DGLG mechanism effectively enhances the layer fusion process by clustering layers with minimal parameter conflicts into the same group.

\noindent \textbf{Effect of the Differential-based Layer Fusion Strategy.} To evaluate the effectiveness of the differential-based layer fusion (DBLF) strategy, we compare it with two baseline variants: \textsc{R-One}, which randomly selects one layer from each group as the representative layer, and \textsc{Sum}, which directly performs the addition operation on all layers within each group to generate the representative layer.
As shown in Table~\ref{tab:ablation_technique2}, DBLF consistently outperforms both baselines. On LLaMA2-7B, \textsc{R-One} and \textsc{Sum} show average performance drops of 4.61\% and 1.43\% respectively, compared to DBLF. 
The performance disparity further widens on LLaMA3.1-8B, where \textsc{R-One} and \textsc{Sum} exhibit larger performance gaps of 10.96\% and 3.05\%, respectively. 
These results demonstrate that DBLF can effectively capture and integrate the unique semantic information from layers within each group.

\vspace{-2mm}
\subsection{Analysis}

\textbf{Compatibility with Existing Methods.} We further conduct experiments to validate the compatibility of \textsc{\our} with existing methods. 
We select two representative approaches, FedIT and FedSA-LoRA, to evaluate the impact of integrating \textsc{\our} on model performance and system efficiency.
Table~\ref{tab_Compatibility} shows that the integration of \textsc{\our} consistently yields improvements across multiple evaluation dimensions. 
For example, integrating \textsc{\our} with FedIT on LLaMA2-13B yields a 3.7\% average performance improvement, 2.9$\times$ faster convergence, and a 2.13$\times$ reduction in communication overhead. 
Similar performance gains are also achieved when combining \textsc{\our} with FedSA-LoRA. 
These experimental results demonstrate that \textsc{\our} serves as a versatile framework that can be effectively combined with existing methods while maintaining their inherent advantages.

\definecolor{steelbluev2}{HTML}{DAE8FC}
\definecolor{steelblue}{HTML}{82B0D2}
\definecolor{color3}{HTML}{FEFAE0}
\definecolor{my_green}{HTML}{00CD00}

\begin{table*}[!t]
\vspace{-2mm}
\caption{Evaluation of \textsc{\our}'s compatibility with existing methods.}
\vspace{-1mm}
\label{tab_Compatibility}
\renewcommand{\arraystretch}{0.9} 
  \centering
  \resizebox{1\textwidth}{!}{
    \begin{tabular}{lccccccccc}
    \toprule[1.5pt]
    \multirow{2}{*}{\textbf{Method}}&\multicolumn{5}{c}{\textbf{Close-Ended Benchmark $\uparrow$}} &
    \multicolumn{2}{c}{\textbf{Resource $\downarrow$}}\\
    \cmidrule(lr{3pt}){2-6} \cmidrule(lr{3pt}){7-8}

    & TruthfulQA & MMLU & IFEval & BBH  & \textbf{Average} & \textbf{Time (h)} & \textbf{Comm. (GB)} \\
   
    \midrule[1.5pt]
    &\multicolumn{5}{c}{\textbf{LLaMA2-7B (INT4)}~\cite{llama2}} \\ 
    \midrule[1.5pt]
    FedIT                                     & 47.57 & 42.45 & 31.76 & 39.28 & 40.27 & 2.49 & 5.03 \\
    FedIT+\our      & \textbf{49.86} & \textbf{43.87} & \textbf{33.65} & \textbf{40.79} & \textbf{42.04} (\textcolor{my_green}{{$\uparrow$ 1.77}}) & \textbf{0.83} (\textcolor{my_green}{{$\times$3.00}}) & \textbf{2.36} (\textcolor{my_green}{{$\times$2.13}}) \\
    \midrule
    FedSA-LoRA                                & 48.24 & 42.91 & 32.71 & 39.36 & 40.81 & 2.38 & 2.52 \\
    FedSA-LoRA+\our & \textbf{50.42} & \textbf{44.57} & \textbf{40.92} & \textbf{41.36} & \textbf{44.32} (\textcolor{my_green}{{$\uparrow$ 3.51}}) & \textbf{0.72} (\textcolor{my_green}{{$\times$3.31}}) & \textbf{1.18} (\textcolor{my_green}{{$\times$2.14}})\\

    \midrule[1.5pt]

    &\multicolumn{5}{c}{\textbf{LLaMA2-13B (INT4)}~\cite{llama2}} \\ 
    \midrule[1.5pt]
    FedIT                                     & 52.40 & 55.45 & 40.33 & 46.14 & 48.58 & 6.67 & 8.39 \\
    FedIT+\our      & \textbf{56.84} & \textbf{58.26} & \textbf{45.49} & \textbf{48.52} & \textbf{52.28} (\textcolor{my_green}{{$\uparrow$ 3.70}}) & \textbf{2.30} (\textcolor{my_green}{{$\times$2.90}}) & \textbf{3.93} 
    (\textcolor{my_green}{{$\times$2.13}}) \\
    \midrule
    FedSA-LoRA                                & 55.73 & 57.51 & 43.21 & 46.91 & 50.84 & 6.42 & 4.20 \\
    FedSA-LoRA+\our & \textbf{57.61} & \textbf{59.25} & \textbf{47.63} & \textbf{49.13} & \textbf{53.41} (\textcolor{my_green}{{$\uparrow$ 2.57}}) & \textbf{2.19} (\textcolor{my_green}{{$\times$2.93}}) & \textbf{1.97} (\textcolor{my_green}{{$\times$2.13}}) \\
    \bottomrule[1.5pt]
    \end{tabular}
    }

\vspace{-5mm}
\end{table*}

\begin{figure}[!t]
    \centering
    \vspace{-4mm}
    \begin{minipage}[c]{0.48\textwidth}
        \centering
        \begin{table}[H]
        \caption{Performance analysis of different initial submodel capacities.}
        \label{tab:start_Point}
        \resizebox{\linewidth}{!}{
        \begin{tabular}{cccccccccc}
        \toprule[1.5pt]
        \textbf{Initial}&\multicolumn{5}{c}{\textbf{Close-Ended Benchmark $\uparrow$}} \\
        \cmidrule(lr{3pt}){2-6} \cmidrule(lr{3pt}){7-8}
    
        \textbf{Capacity}& TruthfulQA & MMLU & IFEval & BBH  & \textbf{Average} \\
        \midrule[1.5pt]
    
        &\multicolumn{5}{c}{\textbf{LLaMA3.1-8B (INT4)}~\cite{llama2}} \\ 
        \midrule[1.5pt]
        1  & 52.45 & 66.85 & 56.83 & 70.12 & 61.56 (\textcolor{red}{$\downarrow$ 2.69}) \\   
        2  & 53.87 & 67.31 & 59.45 & 70.50 & 62.78 (\textcolor{red}{$\downarrow$ 1.47}) \\   
        \rowcolor{my_black!40}
        4  & \textbf{55.23} & \textbf{68.42} & \textbf{62.29} & \textbf{71.04} & \textbf{64.25} \\   
        8  & 53.21 & 67.12 & 58.35 & 70.65 & 62.33 (\textcolor{red}{$\downarrow$ 1.92}) \\   
        16 & 51.08 & 65.89 & 54.12 & 70.01 & 60.28 (\textcolor{red}{$\downarrow$ 3.97}) \\  
        32 & 48.79 & 64.49 & 49.75 & 69.33 & 58.09 (\textcolor{red}{$\downarrow$ 6.16})\\
        \bottomrule[1.5pt]
        \end{tabular}
        }
        \end{table}
    \end{minipage}
\hfill 
\begin{minipage}[c]{0.48\textwidth}
    \centering
    \begin{table}[H]
    \caption{Performance analysis under varying submodel growth rates.}
    \label{tab:scaling_factor}
    \renewcommand{\arraystretch}{0.9} 

    \resizebox{\linewidth}{!}{
    \begin{tabular}{cccccccccc}
    \toprule[1.5pt]
    \textbf{Growth}&\multicolumn{5}{c}{\textbf{Close-Ended Benchmark $\uparrow$}} \\
    \cmidrule(lr{3pt}){2-6} \cmidrule(lr{3pt}){7-8}

    \textbf{Rate}& TruthfulQA & MMLU & IFEval & BBH  & \textbf{Average} \\
   
    \midrule[1.5pt]
    &\multicolumn{5}{c}{\textbf{LLaMA2-7B (INT4)}~\cite{llama2}} \\ 
    \midrule[1.5pt]
    \rowcolor{my_black!40}
    2 & \textbf{50.28} & \textbf{44.15} & \textbf{33.97} & \textbf{40.93} & \textbf{42.33} \\
    4 & 47.96 & 42.56 & 29.87 & 38.79 & 39.80 (\textcolor{red}{$\downarrow$ 2.53}) \\
    8 & 45.68 & 40.07 & 25.63 & 36.92 & 37.08 (\textcolor{red}{$\downarrow$ 5.25})\\

    \midrule[1.5pt]
    &\multicolumn{5}{c}{\textbf{LLaMA2-13B (INT4)}~\cite{llama2}} \\ 
    \midrule[1.5pt]
    \rowcolor{my_black!40}
    2 & \textbf{57.19} & \textbf{58.74} & \textbf{46.45} & \textbf{48.70} & \textbf{52.77} \\
    4 & 52.23 & 56.78 & 34.56 & 42.29 & 46.47 (\textcolor{red}{$\downarrow$ 6.3}) \\
    8 & 48.12 & 52.33 & 26.78 & 37.45 & 41.17 (\textcolor{red}{$\downarrow$ 11.6})\\
    \bottomrule[1.5pt]
    \end{tabular}
    }
    \end{table}
\end{minipage}
\vspace{-4mm}
\end{figure}

\textbf{Impact of Initial Submodel Capacity.} We also conduct experiments to investigate how the initial capacity of submodels affects the overall model performance. Specifically, we experiment with LLaMA3.1-8B and set different initial capacities \{1,2,4,8,16,32\}, while maintaining the same total training budget. The submodel capacity also doubles progressively until reaching the full model capacity. Table~\ref{tab:start_Point} shows that the model achieves optimal performance when the initial capacity is set to 4, while either smaller or larger initial capacities result in performance degradation. This phenomenon is analogous to human learning, where starting from either too early (infancy) or too late (adult) may lead to suboptimal outcomes due to premature or delayed cognitive development.

\textbf{Impact of Submodel Growth Rate.} 
Finally, we explore how different submodel growth rates affect overall performance.
Specifically, we experiment with diverse capacity scaling multipliers \{2,4,8\}.
For instance, a multiplier of 4 indicates that the submodel capacity quadruples at each stage until reaching the full capacity.
This generates capacity sequences of \{4$\to$16$\to$32\} for LLaMA2-7B and \{5$\to$20$\to$40\} for LLaMA2-13B.
Table~\ref{tab:scaling_factor} demonstrates that higher growth rates significantly compromise model performance. For LLaMA2-7B, scaling multipliers of 4 and 8 lead to average performance drops of 2.53\% and 5.25\% respectively. The degradation is even more pronounced for LLaMA2-13B, with decreases of 6.3\% and 11.6\%. This performance deterioration can be attributed to abrupt capacity transitions, which may disrupt the construction of the knowledge structure. 
This phenomenon mirrors natural learning processes, where steady, incremental development typically yields better long-term outcomes compared to the aggressive pursuit of short-term performance gains.

\vspace{-2mm}
\section{Conclusion}
In this paper, we propose \our, an innovative federated fine-tuning approach that reduces 
the resource consumption of LLM fine-tuning via progressive model training. 
Specifically, \our decomposes the fine-tuning process into several developmental stages, where each stage focuses on adapting a submodel with increasing parameter capacity. To efficiently architect the submodels for each stage, \our introduces two novel techniques: a deconfliction-guided layer grouping mechanism and a differential-based layer fusion strategy. Extensive experiments on multiple benchmark datasets demonstrate the effectiveness and efficiency of \our.

\bibliographystyle{plain}
\bibliography{ref}


\newpage
\appendix

\newtheorem{lemma}{Lemma}

\section{Theoretical Convergence Analysis}\label{theoretical_appendix}


In this section, we establish a rigorous theoretical analysis of the convergence properties for \our. Under standard assumptions of smoothness and bounded variance, we characterize both the intra-stage convergence behavior of progressively growing submodels and the inter-stage transition dynamics. Our analysis yields explicit convergence rates and theoretical guarantees, extending the classical federated optimization framework~\cite{li2022federated, li2019convergence, wang2022progfed, wu2025breaking}.

\subsection{Preliminaries and Assumptions}
Let $f(\theta)\!=\!\mathbb{E}_{\xi}[F(\theta;\xi)]$ be the expected loss of the full $L$‐layer model, where $\xi$ denotes a random data sample drawn from the data distribution. At stage $s$, the server constructs a submodel of depth $L_s$ with parameters $\theta^{(s)}\in \mathbb{R}^{d_s}$. Each device $i\in [1,N]$ has access to its local loss function $f_i(\theta^{(s)})\!=\!\mathbb{E}_{\xi\sim\mathcal{D}_i}[F(\theta^{(s)};\xi)]$, and the global objective is
\begin{equation}
    F_s(\theta^{(s)}) \;=\; \frac{1}{N}\sum_{i=1}^N f_i(\theta^{(s)}).
\end{equation}
We make the following standard assumptions, which are widely used in standard federated learning convergence analyses~\cite{li2022federated, li2019convergence}:
\begin{enumerate}
  \item (Smoothness) Each $f_i$ is $L$‐smooth: for all $\theta,\theta'$,
  \begin{equation}
  	\|\nabla f_i(\theta) - \nabla f_i(\theta')\|\leq L\|\theta-\theta'\|.
  \end{equation}
  \item (Unbiased Stochastic Gradients) The stochastic gradient $g_i(\theta;\xi)$ satisfies
  \begin{equation}
  	\mathbb{E}[g_i(\theta;\xi)]=\nabla f_i(\theta), \quad
  	\mathbb{E}\|g_i(\theta;\xi)-\nabla f_i(\theta)\|^2\leq\sigma^2.
  \end{equation}
  \item (Bounded Dissimilarity) There exists $G^2$ such that
  \begin{equation}
  	\frac{1}{N}\sum_{i=1}^N\|\nabla f_i(\theta)\|^2 \leq G^2 + \|\nabla F_s(\theta)\|^2.
  \end{equation}
\end{enumerate}

\subsection{Per‐Stage Convergence}
We first analyze the convergence of the federated fine‐tuning process at a fixed stage $s$, where $\theta^{(s)}_t$ denotes the global submodel after $t$ communication rounds. At each round, participating devices perform $K$ local gradient decent steps with learning rate $\eta$ and then average the updates. Under the assumptions above, classical results for FedAvg~\cite{li2019convergence} yield:
\begin{equation}
\label{eq:fedavg_rate}
  \frac{1}{T}\sum_{t=0}^{T-1}\mathbb{E}\|\nabla F_s(\theta^{(s)}_t)\|^2
  \;\le\; \frac{2\bigl(F_s(\theta^{(s)}_0)-F_s^*\bigr)}{\eta K T}
      + \frac{L\eta\sigma^2}{N} + \frac{12L^2\eta^2 K^2 G^2}{T},
\end{equation}
where $F_s^*=\min_{\theta}F_s(\theta)$. In particular, setting $\eta=O\bigl(1/\sqrt{T}K\bigr)$ balances the first and third terms, giving
\begin{equation}
  \frac{1}{T}\sum_{t=0}^{T-1}\mathbb{E}\|\nabla F_s(\theta^{(s)}_t)\|^2
  = O\Bigl(\frac{1}{\sqrt{T}K}\Bigr) + O\Bigl(\frac{1}{\sqrt{T}N}\Bigr).
\end{equation}
Thus, to achieve an $\varepsilon$‐stationary solution, it suffices that
$T = O\bigl(1/(\varepsilon^2 K)\bigr)$ and $N \ge O(1/\varepsilon^2)$.

\subsection{Knowledge Transfer}
At the end of stage $s$, submodel parameters $\theta^{(s)}_T$ are fused back into the full model via representative layer updates (Section~\ref{sec_Knowledge_inheritance}). Concretely, for each group $\mathrm{g}_n$, we update the LoRA parameters of layers in the full model by
\begin{equation}
   \theta_{j}^{\text{new}}
   = \theta_{\text{anchor}}^{(s)}
   + \beta\sum_{k\in \mathrm{g}_n}\bigl(\theta^{(s)}_{k}-\theta_{\text{anchor}}^{(s)}\bigr),
   \quad j\in \mathrm{g}_n,
\end{equation}
where $\theta_{\text{anchor}}^{(s)}$ denotes the anchor parameter in group $\mathrm{g}_n$. By design, this update preserves the submodel optimum while initializing the next stage submodel $\theta^{(s+1)}_0$ close to $\theta^{(s)}_T$. We quantify this closeness in the following lemma:

\begin{lemma}
Under the layer fusion strategy of Section~\ref{sec_LAIF}, the initialization error for stage $s+1$ satisfies
\begin{equation}
  \bigl\|\theta^{(s+1)}_0 - \theta^{(s)}_T\bigr\|
  \;\le\; C\beta \sum_{n=1}^{L_s}
    \sum_{j,k\in \mathrm{g}_n}\|\theta^{(s)}_j - \theta^{(s)}_k\|,
\end{equation}
for some constant $C>0$ depending on grouping size. Moreover, since layers in each group share high similarity by construction (cf.\ Equation~\eqref{Eq_Similarity}), this bound scales as
\begin{equation}
  \bigl\|\theta^{(s+1)}_0 - \theta^{(s)}_T\bigr\|
  = O\bigl(\beta\,\delta_s\bigr),
  \quad \delta_s = \max_{j,k\in \mathrm{g}_n} \|\theta^{(s)}_j - \theta^{(s)}_k\|.
\end{equation}
\end{lemma}
\begin{proof}
We derive the bound through the triangular inequality and the definition of the representative layer as follows:
\begin{align*}
  \bigl\|\theta^{(s+1)}_0 - \theta^{(s)}_T\bigr\|
  &\;=\;
  \Bigl\| 
    \theta_{\text{anchor}}^{(s)}
    + \beta\sum_{j\in \mathrm{g}_n}(\theta^{(s)}_j - \theta_{\text{anchor}}^{(s)})
    - \theta^{(s)}_{\text{anchor}}
    - \sum_{k\in \mathrm{g}_n}\bigl[\theta^{(s)}_k - \theta_{\text{anchor}}^{(s)}\bigr]
  \Bigr\| \\
  &= \Bigl\|\bigl(\beta-1\bigr)\sum_{j\in \mathrm{g}_n}
       \bigl(\theta^{(s)}_j - \theta_{\text{anchor}}^{(s)}\bigr)\Bigr\|
  \;\le\; |\beta-1|\,|\mathrm{g}_n|\max_{j,k}\|\theta^{(s)}_j - \theta^{(s)}_k\|.
\end{align*}
The desired bound follows by summing over all groups.
\end{proof}

\subsection{Overall Convergence Across $S$ Stages}
We now combine per‐stage convergence with the initialization error to bound the suboptimality of the final model $\theta^{(S)}_T$. Let $\Delta_s = F_s(\theta^{(s)}_0) - F_s^*$ denote the optimality gap at stage $s$. From Equation~\eqref{eq:fedavg_rate}, after $T_s$ rounds at stage $s$,
\begin{equation}
  \Delta_s \;-\; \Delta_{s+1}
  \;\ge\;
  \underbrace{\frac{\eta_s K_s}{2T_s}
    \sum_{t=0}^{T_s-1}\mathbb{E}\|\nabla F_s(\theta^{(s)}_t)\|^2}_{\text{descent}}
  - \underbrace{L\eta_s^2K_s^2 G^2}_{\text{variance}}.
\end{equation}
Accounting for the initialization shift $\|\theta^{(s+1)}_0-\theta^{(s)}_T\|$ and telescoping over $s=1,\dots,S$, we obtain
\begin{equation}
  F_S(\theta^{(S)}_T) - F_1(\theta^{(1)}_0)
  \;\le\; -\sum_{s=1}^S\Bigl[
    \frac{\eta_s K_s}{2T_s}\sum_{t}\|\nabla F_s\|^2 - L\eta_s^2K_s^2 G^2
    - O(\beta\,\delta_s)\Bigr].
\end{equation}
By selecting step‐sizes $\eta_s=O(1/\sqrt{T_s}K_s)$, communication rounds $T_s=O(1/(\varepsilon^2K_s))$, and ensuring $\beta\delta_s=O(\varepsilon^2)$ via sufficiently fine layer grouping (i.e., high intra‐group similarity), we guarantee that the global model reaches an $\varepsilon$‐stationary point of the full objective within
\begin{equation}
  \sum_{s=1}^S T_s \;=\; O\Bigl(\sum_{s=1}^S \frac{1}{K_s\varepsilon^2}\Bigr)
  \;=\; O\Bigl(\frac{1}{\varepsilon^2}\Bigr)
\end{equation}
communication rounds. This convergence rate matches that of end-to-end FedAvg up to constant factors, thereby establishing the theoretical efficiency of \our. This completes the proof of convergence for \our.

\paragraph{Conclusion.}  In summary, \our retains the convergence guarantees of FedAvg under nonconvex objectives while distributing the computational load over multiple lightweight stages.  The cross‐stage knowledge transfer ensures that the optimization trajectory remains close to a local optimum as the model capacity grows, and our quantitative analysis elucidates the efficiency of this transfer.

\section{Additional Implementation Details}\label{appendix_implement}

Our \textsc{\our} is implemented using PyTorch with the support of HuggingFace Transformers library~\cite{wolf2019huggingface} for model and dataset management. Following the experimental setup of OpenFedLLM~\cite{ye2024openfedllm}, we randomly distribute the Alpaca-GPT4 dataset across 20 devices, with 10\% of devices randomly sampled for participation in each training round. Each selected device performs 10 local training iterations with a batch size of 16. The local fine-tuning process utilizes the AdamW optimizer coupled with a cosine learning rate scheduler.  
We adopt a staged learning rate strategy, starting at 1e-6 and incrementing by a factor of 10 at each subsequent stage until reaching 1e-4.
Additionally, we exclusively apply LoRA to $\mathbf{W}_q$ and $\mathbf{W}_v$ matrices in the attention layers~\cite{hu2021lora} and configure the LoRA module with a rank of 32. The maximum sequence length is set to 512 tokens~\cite{ye2024openfedllm}.
The total number of federated fine-tuning rounds is set to 300 for LLaMA2-7B and LLaMA3.1-8B, and increases to 400 for LLaMA2-13B. Moreover, to improve computational efficiency, we apply INT4 quantization~\cite{ye2024openfedllm} to all models and conduct experiments on a single NVIDIA H800 GPU. 
To ensure the reliability of our results, all experiments are repeated multiple times, with the averaged values reported as the final results.

\section{Limitations}\label{appendix_limitation}

While our proposed \textsc{\our} demonstrates superior  performance, several limitations warrant acknowledgment. 
First, our current research primarily focuses on federated learning within a single organization. Extending our method to cross-organizational collaborative scenarios, where addressing incentive mechanisms, trust establishment, and privacy concerns becomes paramount, represents a significant yet valuable direction for future investigation. Second, although our approach substantially reduces computational requirements compared to traditional methods, the overall environmental footprint of training LLMs remains considerable. Future work should more comprehensively quantify the carbon emission reductions achieved through our developmental paradigm and explore additional algorithmic and system-level optimizations to further minimize environmental impact.

\section{Broader Impacts}\label{appendix_broader_Impacts}

\textbf{Positive Impacts.} \textsc{\our} demonstrates significant potential in reducing computational overhead during LLM fine-tuning, leading to substantial energy savings and environmental benefits. This efficiency gain makes LLM adaptation more accessible to researchers and organizations with limited computational resources. Additionally, the accelerated convergence achieved through our method not only shortens the training cycle but also enables more rapid deployment and iteration of AI models, potentially facilitating faster progress in various AI applications.

\textbf{Negative Impacts.} While \textsc{\our} represents an advancement in efficient model training, we acknowledge the broader ethical considerations inherent in AI development. However, we do not identify any direct negative impacts specific to our method beyond those generally associated with machine learning and AI technologies. As with any AI advancement, we encourage responsible implementation and careful consideration of potential applications.

\end{document}